\documentclass{article}

\usepackage{PRIMEarxiv}

\usepackage[utf8]{inputenc} 
\usepackage[T1]{fontenc}    
\usepackage{hyperref}       
\usepackage{url}            
\usepackage{booktabs}       
\usepackage{amsfonts}       
\usepackage{nicefrac}       
\usepackage{microtype}      
\usepackage{lipsum}
\usepackage{fancyhdr}       
\usepackage{graphicx}       
\graphicspath{{media/}}     

\usepackage[authoryear,round]{natbib}

\usepackage{hyperref}
\usepackage{url}
\usepackage{CJKutf8}
\usepackage{enumitem} 
\usepackage{amsthm}
\usepackage{wrapfig}
\usepackage{booktabs}
\usepackage{multirow}
\usepackage{tabularx}
\usepackage{array}
\usepackage[table]{xcolor}
\usepackage{booktabs}
\usepackage{tabularx}
\usepackage{array}
\usepackage{caption}
\usepackage{graphicx}
\usepackage{xcolor}
\newtheorem{theorem}{Theorem}[section]

\usepackage{subcaption}
\usepackage{mathtools}

\newtheorem{definition}{Definition}[section]

\NewDocumentCommand{\tjqq}{+m}{%
  \begin{CJK}{UTF8}{gbsn}{\color{black}#1}\end{CJK}
}

\usepackage{float} 

\usepackage{booktabs}
\usepackage[table]{xcolor}
\usepackage{makecell}

\definecolor{improvegreen}{RGB}{0,128,0}
\definecolor{oursgray}{RGB}{232,232,232}

\newcommand{\gain}[1]{\textcolor{improvegreen}{\scriptsize$_{+#1}$}}
\usepackage{multirow}
\usepackage{tabularx}
\usepackage{array}
\usepackage[table]{xcolor}
\usepackage{booktabs}
\usepackage{tabularx}
\usepackage{array}
\usepackage{caption}
\usepackage{graphicx}
\usepackage{xcolor}

\usepackage{graphicx}
\usepackage{xcolor}
\usepackage{tikz}

\usepackage{graphicx}
\usepackage{xcolor}
\usepackage{tikz}
\usepackage{amssymb}
\usepackage{graphicx}
\usepackage{xcolor}
\usepackage{tikz}
\usepackage{amssymb}

\usepackage{graphicx}
\usepackage{xcolor}
\usepackage{tikz}
\usepackage{amssymb}

\newcommand{\resultcell}[4]{%
  \begin{minipage}[t]{0.155\textwidth}
    \centering
    \vbox{%
      \offinterlineskip
      \hbox{%
        \includegraphics[
          width=\linewidth,
          height=\linewidth,
          trim={#1},
          clip
        ]{#2}%
      }%
      \kern0.6mm
      \hbox to \linewidth{%
        \hfil
        {\fontsize{5.5}{6}\selectfont
         #3\hspace{0.35mm}#4}%
        \hfil
      }%
    }%
  \end{minipage}%
}

\title{Mitigating Capability Degradation in LLM Self-Improvement via Representation Geometry Regularization}

\author{
  Jiaqi Tian \\
  School of Artificial Intelligence \\
  Jilin University \\
  China \\
  \texttt{tjq26@mails.jlu.edu.cn}
  \And
  Ting Long\thanks{Corresponding author.} \\
  School of Artificial Intelligence \\
  Jilin University \\
  China \\
  \texttt{longting@jlu.edu.cn}
}

\begin{document}
\maketitle
\vspace{-5pt}
\begin{abstract}
On-Policy Self-distillation is a self-improvement paradigm that enables large language models (LLMs) to improve their capabilities by leveraging the same model as both teacher and student under different conditioning contexts.
However, we find that OPSD methods can progressively degrade previously acquired capabilities: queries that are correctly responded at earlier training stages may become incorrectly responded as optimization proceeds, even the overall performance continues to improve. Further analysis reveals that this degradation coincides with a deterioration of the geometry of hidden representations, characterized by reduced intra-cluster compactness and diminished inter-cluster separation. Motivated by this observation, we propose Representation Geometry Regularization (RGR), which preserves representation geometry during post-training by encouraging compactness within cluster and separation across cluster. RGR can be seamlessly integrated into existing OPSD methods without altering their core optimization procedures. Extensive experiments on multiple OPSD methods and benchmarks show that RGR generally alleviates capability degradation while retaining capability improvements. Our code is available in \url{https://github.com/supermachinelearner/RGR}

\end{abstract}
\vspace{-0.2cm}
\section{Introduction}
\vspace{-0.2cm}

\begin{wrapfigure}{r}{0.5\linewidth}   
  \centering
  \vspace{-5pt}
  \includegraphics[width=\linewidth]{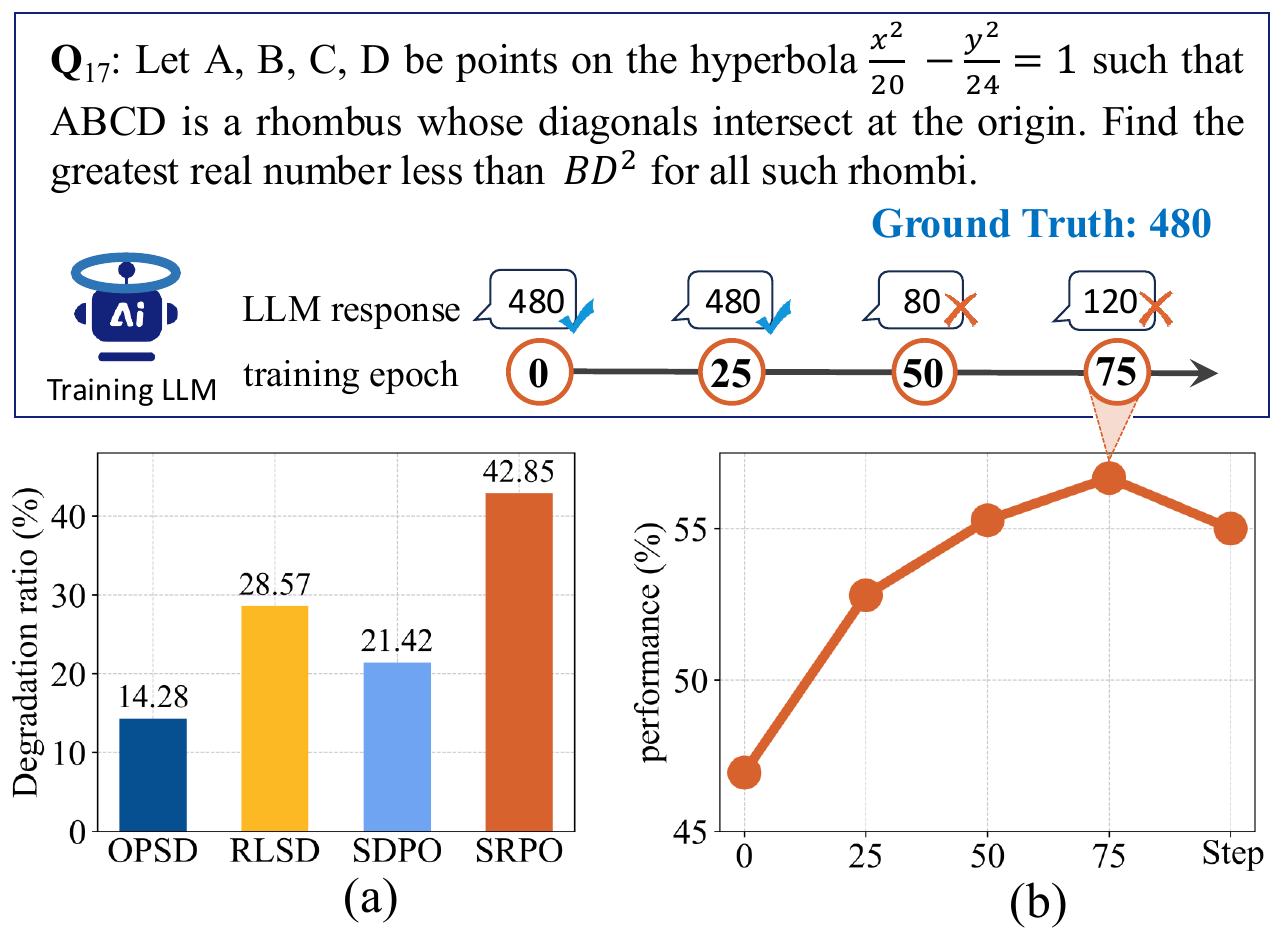}
  \vspace{-20pt}
  \caption{(a) The capability degradation ratio of existing methods; (b) Although the overall performance improves during training, some problems that were previously responded correctly become incorrectly answered as optimization progresses.}
  \vspace{-25pt}
  \label{fig:intro}
\end{wrapfigure}
On-Policy Self-Distillation (OPSD)  \citep{zhao2026self,hubotter2026reinforcement,yang2026self,li2026unifying}, is a self-improvement paradigm for the post-training of large language models (LLMs)  \citep{brown2020language,vaswani2017attention,yang2025qwen3,zeng2026glm,xiao2026mimo,xu2026deepseek}, which enables LLMs to improve their capabilities through self-knowledge distillation.
Unlike conventional knowledge distillation\citep{hinton2015distilling,sanh2019distilbert}, which relies on a separately trained teacher, OPSD uses the same LLM as both teacher and student under different conditioning contexts. 
Specifically, the student model generates responses to the original query, while teacher model, which share the same parameters with the student model, conditioned on additional privileged information (e.g., verified reasoning traces) to provide supervision for these generated responses.

Although recent studies \citep{zhao2026self,hubotter2026reinforcement,yang2026self,li2026unifying,wang2026trace} have empirically demonstrated that optimizing the LLM with OPSD can outperform conventional post-training methods such as SFT \citep{wei2021finetuned} and GRPO \citep{shao2024deepseekmath}, benefiting from the dense and fine-grained token-level supervision of OPSD. We observe a progressive reasoning capability degradation phenomenon during self-improvement in existing OPSD methods. As shown in Figure \ref{fig:intro}(b), a query that is responded correctly at earlier training stages gradually becomes incorrectly responded as training proceeds, even as overall performance continues to improve.

Furthermore, we examine several representative OPSD methods and find that this phenomenon consistently occurs across these methods, albeit to varying degrees, as it is shown in Figure \ref{fig:intro} (a). Such degradation may undermine the stability of self-improvement.

To better understand this phenomenon, we conduct further analysis, and the results reveal that this capability degradation is accompanied by a progressive deterioration in the geometry of hidden representations. When a query is answered correctly at earlier stages of training, the hidden representations of its generated responses exhibit a relatively well-structured clustering pattern, characterized by larger inter-cluster distances and clearer cluster boundaries. As training proceeds and the response to the same query gradually becomes incorrect, the relative distances between clusters shrink and the cluster boundaries become increasingly indistinct. This observation suggests that preserving the geometric structure of the representation space may be important for maintaining previously acquired capabilities during self-improvement.

Based on this finding, we theoretically show that preserving intra-cluster compactness and inter-cluster separation in the hidden representation space helps maintain the capabilities that the LLM has already acquired. Based on this theoretical insight, we develop  Representation Geometry Regularization (RGR), which constrains the evolution of hidden representations during self-improvement. RGR incorporates two geometry-preserving constraints called intra-cluster compactness and inter-cluster separation into the OPSD methods, enabling the model to improve targeted capabilities while mitigating the degradation of previously acquired ones. Experiments across four existing SD models and three widely used benchmarks show that RGR substantially reduces correct-to-incorrect transitions during training while retaining its self-improvement gains. 

In summary, our contributions are summarized as follows:
\begin{itemize} [leftmargin=2em]
    \item We uncover an overlooked failure mode, where OPSD methods can degrade previously acquired capabilities, causing previously correct predictions to become incorrect.
    \item We identify representation geometry distortion as a key factor underlying this degradation, and theoretically characterize the importance of hidden representation geometry for prediction stability.
    \item We propose RGR, which incorporates geometry-preserving constraints into OPSD methods to enable self-improvement while reducing capability degradation.
\end{itemize}
\vspace{-0.2cm}
\section{Related works}
\vspace{-0.2cm}
On-policy self-distillation (OPSD) enables LLMs to improve their capabilities without requiring an external model as a teacher. Recently, it has attracted significant attention as a promising paradigm for scalable and efficient LLM self-improvement.
Based on how self-teacher signals are constructed and how external feedback is incorporated, we organize existing OPSD methods into privileged-information-driven methods and external-feedback-driven methods.The privileged-information-driven methods construct self-teacher signals using additional privileged information that is unavailable to the student policy at inference time. As a representative method, OPSD \citep{zhao2026self} uses the same model to construct self-teacher signals and optimize the student policy under different conditioning contexts: the self-teacher signals are constructed from both the question and a verified reasoning trace, whereas the student policy is conditioned only on the question.
The external-feedback-driven methods
retain the self-teacher structure while incorporating environmental feedback, verification results, or outcome rewards to construct the supervision for policy updating. Within this category, feedback can be used to construct distillation targets, select training objectives, or modulate token-level advantages. SDPO \citep{hubotter2026reinforcement} conditions the current model on textual or reward-derived feedback to form a self-teacher, whose next-token predictions serve as distillation targets for the student policy. SRPO \citep{li2026unifying}routes successful trajectories to reward-driven GRPO updates and failed trajectories to entropy-weighted, token-level self-distillation against a feedback-conditioned teacher distribution. RLSD \citep{yang2026self} uses verifiable outcomes to determine the direction of policy updates and token-level teacher–student discrepancies to modulate their magnitude, integrating self-distillation into reward-driven learning.

\section{Preliminaries}
\vspace{-0.05cm}
Given a query $q$, an LLM $\pi_\theta$ typically generates a response $y=(y_1,\ldots,y_T)$ according to $y_t \sim \pi_\theta(\cdot \mid q,y_{<t})$,  
where $y_{<t}=(y_1,\ldots,y_{t-1})$ denotes the previously generated tokens.
OPSD methods improve the LLM by using the same model under different conditioning contexts as the teacher and student to conduct knowledge distillation. Specifically, the student is designed to generate an on-policy response based only on the original query, like the LLM, \textit{i.e.},
$y \sim \pi_\theta(\cdot\mid q)$.
While the teacher is designed to generate the response with additional privileged information $z$, such as a verified reasoning trace, $\pi_\theta (\cdot\mid q,z,y_{<t})$.
The OPSD methods optimize the model by minimizing the divergence between the teacher and student token distributions along the student's trajectory:
\vspace{-0.2cm}
\begin{equation}
\mathcal{L}_{\mathrm{OPSD}}(\theta)
=
\mathbb{E}_{q \sim \mathcal{D},\, y \sim \pi_{\theta}(\cdot \mid q)}
\left[
\sum_{t=1}^{T}
D\left(
\pi_{\theta}(\cdot \mid q, z, y_{<t})
\,\middle\|\,
\pi_{\theta}(\cdot \mid q, y_{<t})
\right)
\right].
\end{equation}
\vspace{-0.3cm}

where $D(\cdot|\cdot)$ denotes a divergence measure between the teacher and student token distributions. $D$ can be instantiated by forward KL divergence  \citep{kullback1951information}, reverse KL divergence \citep{minka2005divergence}, Jensen--Shannon divergence \citep{lin1991divergence}, or other suitable distribution-matching objectives.

Since the teacher and student share the underlying model parameters, optimizing $\mathcal{L}_{\mathrm{OPSD}}$ updates the LLM itself. Consequently, the model iteratively improves its responses by distilling the additional information available to the teacher into its ordinary, query-conditioned behavior.


\section{The Geometry of Capability Degradation}
\label{sec:cap_degradation}
\vspace{-0.05cm}
\subsection{Progressive Capability Degradation}

\begin{wrapfigure}{r}{0.48\textwidth}
    \centering
    \vspace{-5pt}
    \includegraphics[width=\linewidth]{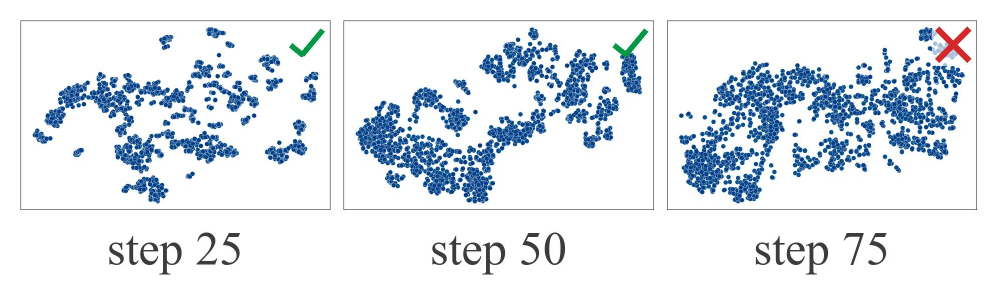}
    \caption{Hidden-state representations and response correctness across training epochs. The query is answered correctly at epochs 25 and 50, but incorrectly at epoch 75.}
    \label{fig:cap_degradation}
    \vspace{-10pt}
\end{wrapfigure}
As discussed previously, although OPSD methods continuously improve the overall performance in the training, some queries that are correctly responded at earlier training stages become incorrectly responded as training proceeds. 
To better understand this phenomenon, for a fixed query, we generate a response and visualize the hidden representations corresponding to the generated response every 25 training epochs. We observe that changes in the geometry of the hidden representations consistently coincide with changes in response correctness.
As illustrated in Figure.~\ref{fig:cap_degradation},
when the model generates a correct response at an earlier training stage (Epoch 25), the corresponding hidden representations exhibit a relatively structured geometric organization, with clearly distinguishable clusters and relatively large distances between clusters. As training proceeds and the response becomes incorrect, the relative distances between clusters gradually decrease, while the boundaries between clusters become increasingly indistinct (Epoch 75).
We hypothesize that as a query transitions from being correctly to incorrectly responded, the reasoning process in the generated response gradually becomes less coherent. When the query is answered correctly, the reasoning tends to exhibit a clear and consistent logical structure, which is reflected in the hidden representation space as distinct and well-separated clusters corresponding to different reasoning patterns. As the reasoning becomes increasingly disorganized, these patterns become less distinguishable, causing the corresponding representations to become more entangled and eventually leading to an incorrect response.
Interestingly, this phenomenon is consistently observed in existing OPSD methods, such as OPSD\citep{zhao2026self}, SDPO\citep{hubotter2026reinforcement}, RLSD\citep{yang2026self}, SRPO\citep{li2026unifying}.
These observations suggest a strong connection between the geometry of the hidden representation and the stability of the model's acquired capabilities during self-improvement.

\subsection{Theoretical Analysis for Capability Degradation}
\label{sec:theory_geometry}

To further substantiate the relationship between hidden representation geometry and response correctness, we conduct a theoretical analysis.

Suppose the LLM undergoing on-policy self-distillation is denoted by $\pi_{\theta}(\cdot \mid q, y_{<t})$,
where $\pi_{\theta} \in\Pi_{\mathrm{OPSD}}$, and $\Pi_{\text{OPSD}}$ denote the policy set trained using OPSD. 
 
Let $\pi_{\theta*}(\cdot \mid q, z, y_{<t})$ denote the optimal teacher distribution conditioned on the privileged information $z$, and let $\pi_{\theta*}(\cdot \mid q, y_{<t})$ denote the corresponding optimal student distribution without access to $z$. We define the excess distillation risk as follows.

\begin{definition}[The Excess Distillation Risk] 
The excess distillation risk of the student model $\pi_{\theta}(\cdot \mid q, y_{<t})$ is defined as the difference between its distillation risk with respect to the optimal teacher and that achieved by the optimal student, i.e.,
\vspace{-0.1cm}
\label{def:excess_risk}

\begin{equation}
\begin{aligned}
\mathcal{E}_{\text{KD}}(\pi_{\theta})=\mathbb{E}_{q \sim \mathcal{D}, y \sim \pi_{\theta}(\cdot \mid q)}
\Bigl[
\sum_{t=1}^{T}
D\left(
\pi_{\theta^*}(\cdot \mid q, z, y_{<t})
\,\middle\|\,
\pi_{\theta}(\cdot \mid q, y_{<t})
\right) \\
-D\left(
\pi_{\theta^*}(\cdot \mid q, z, y_{<t})
\,\middle\|\,
\pi_{\theta^*}(\cdot \mid q, y_{<t})
\right)
\Bigr].
\end{aligned}
\end{equation}
\end{definition}

This quantity measures the gap between the distillation risk of the current student and that of the optimal student, with a smaller value indicating a closer approximation to the optimal.

Let $\hat{\mathcal H}=\{\hat{\mathbf h}_i\}_{i=1}^{T}$ and $\mathcal H^*=\{\mathbf h_i^*\}_{i=1}^{T}$ denote the hidden representations produced by the current teacher $\pi_{\theta}(\cdot\mid q,z,y_{<t})$ and the optimal teacher $\pi_{\theta^*}(\cdot\mid q,z,y_{<t})$, respectively. 
Their midpoint representations are denoted by $\mathcal H^\dagger=\{\mathbf h_i^\dagger\}_{i=1}^{T}$, where $\mathbf h_i^\dagger=\frac{(\hat{\mathbf h}_i+\mathbf h_i^*)}{2}$. Suppose that $\mathcal H^\dagger$ is partitioned into $K$ clusters, denoted by $\mathcal C^\dagger=\{C_1^\dagger,\ldots,C_K^\dagger\}$. 
The center of is cluster $C_k^\dagger$ is $\mu_k^\dagger = \frac{1}{|C_k^\dagger|} \sum_{\mathbf{h}_i^\dagger\in C_k^\dagger} \mathbf{h}_i^\dagger$. 
Let $\mathcal H=\{\mathbf h_i\}_{i=1}^{T}$ denote the student hidden representations corresponding to the current-teacher representations $\{\hat{\mathbf h}_i\}_{i=1}^{T}$. For each cluster $C_k^\dagger$, we define $\mathcal X_k=\{\mathbf h_i\mid\mathbf h_i^\dagger\in C_k^\dagger\}$. That is, if the midpoint $\mathbf h_i^\dagger$ associated with a current-teacher representation $\hat{\mathbf h}_i$ belongs to $C_k^\dagger$, its corresponding student representation $\mathbf h_i$ is assigned to $\mathcal X_k$. The center of is cluster $\mathcal X_k$ is $\mu_k^{\mathcal{X}}=\frac{1}{|\mathcal{X}_k|}\sum_{h_i\in\mathcal{X}_k}\mathbf{h}_i$. Let $\mu=\frac{1}{T}\sum_i^T\mathbf{h}_i$ denote the mean of hidden representations. Based on the definition, we have

\begin{theorem}[Excess Risk Upper Bound]
\label{thm:excess_risk}

\tjqq{
\begin{equation}
\begin{aligned}
\mathcal{E}_{\text{KD}}(\pi_{\theta}) \leq \frac{2||M||_2}{N|\mathcal{V}|}\sum_N\sup_{\pi_{\theta}\in \Pi_\mathrm{OPSD}}|\Delta|[\sqrt{\frac{Td}{2}+\frac{1}{2}\sum_{|\mathcal{C}^\dagger|}\sum_{{\mathbf{h}_i}\in \mathcal{X}_k}||\mathbf{h}_i-\mu_k^{\mathcal{X}}||_2^2-\frac{1}{2}\sum_{|\mathcal{C}^\dagger|}|\mathcal{C}^\dagger_k|||\mu_k^{\mathcal{X}}-\mu||^2_2}+\\
\sqrt{\sum_{|\mathcal{C}^\dagger|}|\mathcal{C}_k^\dagger|||\mu_k^{\mathcal{X}}-\mu_k^{\dagger}||^2_2+\sum_{|\mathcal{C}^\dagger|}\sum_{{\mathbf{h}_i^{\dagger}}\in \mathcal{X}_k}||\mathbf{h}_i^{\dagger}-\mu_k^{\dagger}||_2^2}]
\end{aligned}
\end{equation}}
where  $N$ denote the number of samples, $|\Delta|=$  $\sqrt{\sum_i||\hat{\mathbf{h}_i}-{\mathbf{h}_i}^*||_2^2}$, $M=W^T(I_{|\mathcal{V}|}-\frac{1}{|\mathcal{V}|}\mathbf{1}_{|\mathcal{V}|}\mathbf{1}_{|\mathcal{V}|}^T)W$. The detail proof could be found in Appendix \tjqq{\ref{proof:excess_risk}}
\end{theorem}

From Theorem \ref{thm:excess_risk}, we observe that the excess risk upper bound is affected by both intra-cluster compactness and inter-cluster separation of the learned representations. Specifically, reducing the intra-cluster variance term 
$\|\mathbf{h}_i-\mu_k^{\mathcal{X}}\|_2^2$
and increasing the inter-cluster dispersion term
$\|\mu_k^{\mathcal{X}}-\mu\|_2^2$
can tighten the excess risk upper bound. Meanwhile, the term
$\|\mu_k^{\mathcal{X}}-\mu_k^{\dagger}\|_2^2$
reflects the deviation between the learned student representation and the teacher and optimal teacher representation. Therefore, preventing capability degradation during OPSD requires not only maintaining the alignment with the teacher representation, but also preserving compact intra-cluster compactness and clear inter-cluster separation in the representation space.

\section{Methodology}

As discussed previously, we observe that increasingly blurred boundaries between hidden representation clusters coincide with the progressive degradation of model capabilities during OPSD training. Motivated by this observation, we propose Representation Geometry Regularization (RGR) to preserve clear boundaries between hidden representation clusters. To achieve this goal, we consider two complementary geometric properties: (1) intra-cluster compactness, where representations within each cluster remain sufficiently compact, and (2) inter-cluster separation, where different clusters maintain sufficient separation from each other. In the following, we discuss how to preserve the two properties in detail and derive the corresponding regularization objectives.

\subsection{Intra-cluster compactness}

According to previous works , the intra-cluster compactness of hidden representations can be evaluated by the Within-Cluster Sum of Squares (WCSS)\citep{lloyd1982least}, defined as the sum of squared distances between each hidden representation and the center of its corresponding cluster. 
However, directly minimizing WCSS may lead to excessive compactness within each cluster. In the extreme case, hidden representations within the same cluster may become overly similar, reducing the discriminative information encoded in the representation space. Therefore, we introduce the teacher model as a reference and encourage the WCSS of the student representations within each cluster to remain close to that of the teacher representations.

Given the query $q$, the hidden representations of the response generated by the student model are denoted as $\{\mathbf{h}_1,\ldots,\mathbf{h}_T\}$, while the corresponding hidden representations of the reference model are denoted as $\{\mathbf{h}_1^{\text{ref}},\ldots,\mathbf{h}_T^{\text{ref}}\}$.
\tjqq{Suppose that a cluster $C^\mathrm{ref}_k$ in the reference model's representation space corresponds to a cluster $C_k^\mathrm{align}$ in the student model's representation space, with cluster centers $\mathbf{u}_k^\mathrm{ref}$ and $\mathbf{u}_k^\mathrm{align}$, respectively.} To align the intra-cluster compactness of the student model with that of the reference model, the learning target for intra-cluster compactness is defined as:

\begin{equation}
\label{eq:intra_ini}
\mathcal{L}_{\mathrm{intra}}=
\sum_{|\mathcal{C}_\mathrm{ref}|} \left|
\sum_{
\mathbf{h}_{i}\in C_{k}^\mathrm{align},
}
\left\|
\mathbf{h}_{i}-\boldsymbol{\mu}_{k}^\mathrm{align}
\right\|_2^2
-
\sum_{\mathbf{h}_i^\mathrm{ref}\in C_k^\mathrm{ref}}
\left\|
\mathbf{h}_i^\mathrm{ref}-\boldsymbol{\mu}_k^\mathrm{ref}
\right\|_2^2
\right|,
\end{equation}

where $|\mathcal{C}_\mathrm{ref}|$ denotes the number of clusters.

However, directly optimizing Eq.~\ref{eq:intra_ini} is difficult. To address this, we stack the hidden representations of the student model as $\mathbf{H}=
[\mathbf{h}_1,\ldots,\mathbf{h}_T] \in\mathbb{R}^{T\times d}$, where each hidden representation is the row vector of $\mathbf{H}$. 
Similarly, we stack the hidden representations of the student model as $\mathbf{H}_\mathrm{ref}=
[\mathbf{h}_1^\mathrm{ref},\ldots,\mathbf{h}_T^\mathrm{ref}]$, where each hidden representation is the row vector of $\mathbf{H}_\mathrm{ref}$. 
Under the above definitions and assumptions, we have:

\tjqq{\begin{theorem}[Intra-cluster Compactness Upper Bound]
\label{thm:cluster}
The upper bound of the intra-cluster compactness learning target is 

\begin{equation}
    \mathcal{L}_\mathrm{intra}\leq (T-|\mathcal{C}_\mathrm{ref}|)^{\frac{1}{2}}||\mathbf{H}\mathbf{H}^\top-\mathbf{H}_\mathrm{ref} \mathbf{H}_\mathrm{ref}^T||_F,
\end{equation}
\end{theorem}

where $|\mathcal{C}_\mathrm{ref}|$ denotes the number of clusters of the reference model.}

The details of the proof can be found in Appendix \ref{proof:cluster}.

\subsection{Inter-cluster separation}


We construct a Laplacian matrix\citep{von2007tutorial} from the pairwise relationships among all hidden representations and use its low-end spectrum, corresponding to the smallest eigenvalues of the Laplacian matrix, to evaluate the separability between clusters. However, directly maximizing these low-end eigenvalues is undesirable, as excessively large inter-cluster separation can over-expand the representation space and distort the relative geometric relationships among clusters.
Therefore, we use the teacher model as a reference model to encourage the student model to preserve its inter-cluster separation, thereby preventing undesirable geometric distortion during optimization.

Specifically, given the query $q$, the hidden representations of the response generated by the student model are denoted as $\{\mathbf{h}_1,\ldots,\mathbf{h}_T\}$, while the corresponding hidden representations of the reference model are denoted as $\{\mathbf{h}_1^{\text{ref}},\ldots,\mathbf{h}_T^{\text{ref}}\}$. 
We define the Gaussian kernel matrices of the hidden representations of the student model and the reference model as:
\begin{equation}
\begin{aligned}
    A_{ij}&=
    \begin{cases}
    \exp\left(
    -\dfrac{\left\|\mathbf{h}_i-\mathbf{h}_j\right\|_2^2}
    {2\sigma^2}
    \right), & i\neq j,\\
    0, & i=j,
    \end{cases}
    \qquad
    A_{ij}^{ref}&=
    \begin{cases}
    \exp\left(
    -\dfrac{\left\|\mathbf{h}_i^\mathrm{ref}-\mathbf{h}_j^\mathrm{ref}\right\|_2^2}
    {2\sigma^2}
    \right), & i\neq j,\\
    0, & i=j,
    \end{cases}
\end{aligned}
\end{equation}
where $\sigma$ denotes the standard deviation of the Gaussian kernel.

The corresponding degree matrices of the student and reference Gaussian kernel matrices are denoted by $\mathbf{D}$ and $\mathbf{D}^{\mathrm{ref}}$, respectively, where $D_{ii}=\sum_{j=1}^{T}A_{ij}$ and $D_{ii}^{\mathrm{ref}}=\sum_{j=1}^{T}A_{ij}^{\mathrm{ref}}$. The unnormalized Laplacian matrices of the student and reference hidden representations are then defined as $\mathbf{L}=\mathbf{D}-\mathbf{A}$ and $\mathbf{L}^{\mathrm{ref}}=\mathbf{D}^{\mathrm{ref}}-\mathbf{A}^{\mathrm{ref}}$, respectively. 
The corresponding eigenvalue vectors are

\begin{equation}
\begin{aligned}
\boldsymbol{\lambda}
&=
\left[
\lambda_1(\mathbf{L}),\ldots,
\lambda_T(\mathbf{L})
\right]^\top,
\\
\boldsymbol{\lambda}^{\mathrm{ref}}
&=
\left[
\lambda_1(\mathbf{L}^{\mathrm{ref}}),\ldots,
\lambda_T(\mathbf{L}^{\mathrm{ref}})
\right]^\top,
\end{aligned}
\end{equation}

where the eigenvalues are arranged in nondecreasing order. The learning target for inter-cluster separation could be represented by the difference between the eigenvalue vectors  of the student and the reference, that is 

\begin{equation}
    \mathcal{L}_\mathrm{inter} =  \left\|
    \boldsymbol{\lambda}
    -
    \boldsymbol{\lambda}^\mathrm{ref}
    \right\|_2.
\end{equation}

Then, by the Hoffman–Wielandt inequality \citep{hoffman1953variation}, we obtain

\begin{equation}
    \mathcal{L}_\mathrm{inter}
    \leq
    (1+\sqrt{T})||\mathbf{L}-\mathbf{L^{ref}}||_F
\end{equation}
Under the above definitions and assumptions, we have
\begin{theorem}[Inter-cluster Separation Upper Bound] The upper bound of the inter-cluster separation could be induced by applying Taylor expansion and the Newton–Schwarz inequality: 
\label{thm:inter}
\label{}

\begin{equation}
\mathcal{L}_\mathrm{inter}
\leq
\left(
\frac{1+\sqrt{T}}{\sigma}
\right)^2
\left\|
\mathbf{H}\mathbf{H}^{\top}
-
\mathbf{H}_{\mathrm{ref}}\mathbf{H}_{\mathrm{ref}}^{\top}
\right\|_F,
\end{equation}


\end{theorem}
The proof is provided in Appendix~\ref{proof:inter}.

\subsection{Optimization}

Building upon the previous analysis, preserving the intra-cluster compactness and inter-cluster separation of the hidden representation space requires minimizing $\mathcal{L}_{\mathrm{intra}}$ and $\mathcal{L}_{\mathrm{inter}}$. From the above theoretical analysis, we show that both geometric discrepancies can be controlled by the discrepancy between the Gram matrices of the student and reference hidden representations. Specifically, we have
$
\mathcal{L}_{\mathrm{intra}}
\leq
C_{\mathrm{intra}}
\left\|
\mathbf{H}\mathbf{H}^{\top}
-
\mathbf{H}_{\mathrm{ref}}\mathbf{H}_{\mathrm{ref}}^{\top}
\right\|_F
$
and
$
\mathcal{L}_{\mathrm{inter}}
\leq
C_{\mathrm{inter}}
\left\|
\mathbf{H}\mathbf{H}^{\top}
-
\mathbf{H}_{\mathrm{ref}}\mathbf{H}_{\mathrm{ref}}^{\top}
\right\|_F,
$
where $C_{\mathrm{intra}}$ and $C_{\mathrm{inter}}$ are constants independent of the student representations. Therefore, minimizing the Gram matrix discrepancy provides a unified objective for preserving both intra-cluster compactness and inter-cluster separation. We thus define the geometry-preserving objective as

\begin{equation}
\mathcal{L}_{\mathrm{geo}}
=
\left\|
\mathbf{H}\mathbf{H}^{\top}
-
\mathbf{H}_{\mathrm{ref}}\mathbf{H}_{\mathrm{ref}}^{\top}
\right\|_F^2.
\end{equation}
In addition to representation geometry, recent studies \citep{jin2026entropy} on on-policy distillation have highlighted the importance of controlling the entropy of the model output distribution during optimization,which motivates us to explicitly incorporate an entropy-related objective into the optimization process. Accordingly, we jointly preserve the hidden representation geometry and regulate the output distribution entropy during OPSD optimization.
Combining the original OPSD objective with the geometry-preserving and entropy-related objectives, we formulate the overall training objective as
\begin{equation}
\label{equ_ab}
\mathcal{L}
=
\mathcal{L}_{\mathrm{OPSD}}
+
\alpha\mathcal{L}_{\mathrm{geo}}
+
\beta\mathcal{L}_{\mathrm{ent}},
\end{equation}
where $\mathcal{L}_{\mathrm{OPSD}}$ denotes the original on-policy self-distillation objective, $\mathcal{L}_{\mathrm{geo}}$ denotes the proposed geometry-preserving regularization, and $\mathcal{L}_{\mathrm{ent}} = 
\left\|
\mathbf{H}^{\top}\mathbf{H}
-
\mathbf{H}_{\mathrm{ref}}^{\top}\mathbf{H}_{\mathrm{ref}}
\right\|_F^2$
 denotes the entropy-related objective. The coefficients $\alpha$ and $\beta$ control the contributions of the geometry-preserving and entropy-related terms, respectively.

\section{EXEPERIMENT}
\label{gen_inst}

In this section, we conduct experiments to evaluate the performance of RGR. We aim to answer the following research questions: (1) Is RGR effective for improving the performance of self distillation (RQ1)?
(2) How do different components of RGR contribute to the performance gains (RQ2)?
(3) Can RGR alleviate the degradation of previously acquired capabilities during OPSD optimization, and how does it affect the evolution of hidden representations (RQ3)? 
(4) Can RGR generalize to task-transfer settings and alleviate the degradation of previously acquired capabilities when learning new tasks (RQ4)?

\subsection{Experiment setup}

\paragraph{Datasets.}
Following the dataset setup of OPSD \citep{zhao2026self}, we sample up to 30K problem-solution pairs containing chain-of-thought reasoning from the mathematical reasoning subset of OpenThoughts\citep{guha2506openthoughts}for model training.. 
We conduct evaluation on three representative benchmarks: AIME 2024 \citep{aime24}, AIME 2025 \citep{aime25}, and HMMT 2025 \citep{balunovic2026matharena}. AIME 2024 consists of problems from the 2024 American Invitational Mathematics Examination, covering diverse areas of high-school mathematics. AIME 2025 consists of problems from the 2025 American Invitational Mathematics Examination. HMMT 2025 consists of problems from the Harvard-MIT Mathematics Tournament, covering algebra, combinatorics, geometry, and number theory. 

\paragraph{Baseline.}

To assess the effectiveness of RGR, we integrate it into four representative self-distillation methods, namely, OPSD \citep{zhao2026self}, SDPO \citep{hubotter2026reinforcement}, SRPO \citep{li2026unifying}, and RLSD \citep{yang2026self}, and compare each RGR-enhanced method with its original counterpart. These methods differ in how they optimize the self-distillation objective. Specifically, OPSD employs forward KL divergence \citep{kullback1951information}, with the policy update constrained by token-level KL clipping; SDPO uses Jensen–Shannon divergence \citep{lin1991divergence} to align the student with the teacher; SRPO routes successful responses to reward-based reinforcement and failed responses to self-distillation; and RLSD uses environment-derived rewards to determine the direction of policy updates, while reweighting token-level advantages according to the teacher–student log-probability gap to modulate the update magnitude.

\paragraph{ImplementationDetails.}
Our experiments are conducted on two foundation LLMs: Qwen3-1.7B and Qwen3-4B\citep{yang2025qwen3}, with thinking mode enabled throughout training and validation.  
All experiments are performed on four NVIDIA A40 GPUs. All methods use LoRA \citep{hu2021lora} with a rank of 64 and a scaling factor of 128, with the learning rate set to 5 \(\times\) \(10^{-6}\). The teacher model remains frozen throughout training. For our method, 
we fix the parameter $\alpha$ in Eq.(\ref{equ_ab}) at $2.5 \times 10^{-2}$ and the parameter $\beta$ in Eq.(\ref{equ_ab}) at $5 \times 10^{-2}$
across all benchmarks.

\subsection{Main results (RQ1) }

\begin{wraptable}{r}{0.43\linewidth}
  \vspace{-10pt}
  \scriptsize
  \renewcommand{\arraystretch}{1}
  \setlength{\tabcolsep}{2pt}

  \captionsetup{justification=raggedright, singlelinecheck=false}
  \caption{Performance gains from RGR. Best results in bold; RGR improvements in green. }
  \label{tab:main_results}

  \begin{tabularx}{\linewidth}{
    @{}
    c
    l
    >{\centering\arraybackslash}p{1.0cm}
    >{\centering\arraybackslash}p{1.0cm}
    >{\centering\arraybackslash}p{1.0cm}
    >{\centering\arraybackslash}p{1.1cm}
    @{}
  }
  \toprule
  \textbf{Model} & \textbf{Method}
  & \textbf{AIME24} & \textbf{AIME25} & \textbf{HMMT25} & \textbf{Avg.} \\
    \midrule
  \multirow[c]{8}{*}{\rotatebox[origin=c]{90}{Qwen3-1.7B}}
  & OPSD          & 56.67 & 43.06 & 28.89 & 42.87 \\
  & \cellcolor{oursgray}\textbf{OPSD + RGR}
                  & \cellcolor{oursgray}\textbf{59.89} & \cellcolor{oursgray}\textbf{43.60} & \cellcolor{oursgray}\textbf{30.00}
                  & \cellcolor{oursgray}\textbf{44.50}\gain{1.63} \\
  & SDPO          & 38.61 & 28.61 & 21.39 & 29.54 \\
  & \cellcolor{oursgray}\textbf{SDPO + RGR}
                  & \cellcolor{oursgray}\textbf{41.94} & \cellcolor{oursgray}\textbf{31.39} & \cellcolor{oursgray}\textbf{21.94}
                  & \cellcolor{oursgray}\textbf{31.76}\gain{2.22} \\
  & RLSD          & 50.27 & 39.70 & 25.83 & 38.60 \\
  & \cellcolor{oursgray}\textbf{RLSD + RGR}
                  & \cellcolor{oursgray}\textbf{51.39} & \cellcolor{oursgray}39.17 & \cellcolor{oursgray}25.56
                  & \cellcolor{oursgray}\textbf{38.70}\gain{0.10} \\
  & SRPO          & 37.78 & 30.00 & 20.00 & 29.26 \\
  & \cellcolor{oursgray}\textbf{SRPO + RGR}
                  & \cellcolor{oursgray}\textbf{39.44} & \cellcolor{oursgray}29.72 & \cellcolor{oursgray}\textbf{21.67}
                  & \cellcolor{oursgray}\textbf{30.28}\gain{1.02} \\
    \midrule
  \multirow[c]{8}{*}{\rotatebox[origin=c]{90}{Qwen3-4B}}
  & OPSD          & 74.17 & 68.06 & 45.00 & 62.41 \\
  & \cellcolor{oursgray}\textbf{OPSD + RGR}
                  & \cellcolor{oursgray}\textbf{75.56} & \cellcolor{oursgray}\textbf{70.56} & \cellcolor{oursgray}\textbf{47.50}
                  & \cellcolor{oursgray}\textbf{64.54}\gain{2.13} \\
  & SDPO          & 67.22 & 61.39 & 38.61 & 55.74 \\
  & \cellcolor{oursgray}\textbf{SDPO + RGR}
                  & \cellcolor{oursgray}\textbf{72.50} & \cellcolor{oursgray}\textbf{63.61} & \cellcolor{oursgray}\textbf{41.11}
                  & \cellcolor{oursgray}\textbf{59.07}\gain{3.33} \\
  & RLSD          & 75.00 & 65.00 & 44.44 & 61.48 \\
  & \cellcolor{oursgray}\textbf{RLSD + RGR}
                  &\cellcolor{oursgray} 74.72 & \cellcolor{oursgray}\textbf{66.11} & \cellcolor{oursgray}\textbf{44.72}
                  & \cellcolor{oursgray}\textbf{61.85}\gain{0.37} \\
  & SRPO          & 72.22 & 64.16 & 39.17 & 58.52 \\
  & \cellcolor{oursgray}\textbf{SRPO + RGR}
                  & \cellcolor{oursgray}\textbf{73.33} & \cellcolor{oursgray}\textbf{66.67} & \cellcolor{oursgray}\textbf{41.67}
                  & \cellcolor{oursgray}\textbf{60.56}\gain{2.04} \\
\bottomrule
\end{tabularx}
\end{wraptable}

Following OPSD \citep{zhao2026self}, we evaluate each model by generating 12 responses for each problem and report the average accuracy over the 12 responses (Avg@12).

The results are reported in Table~\ref{tab:main_results}. As it is shown in Table~\ref{tab:main_results}, we can observe that: (1) RGR generally improves different methods across both Qwen3-1.7B and Qwen3-4B. OPSD, SDPO, and SRPO consistently benefit from RGR, demonstrating that representation geometry preservation provides complementary benefits to different optimization objectives. Although a few settings show slight decreases, RGR improves the majority of cases, supporting its general effectiveness; (2) The gains vary across different methods, with larger improvements generally observed for OPSD and SDPO but limited gains for RLSD. This can be attributed to their different optimization mechanisms: OPSD and SDPO directly rely on self-distillation signals for dense token-level supervision, while RLSD mainly uses self-distillation to determine the update magnitude and obtains the update direction from RLVR. Consequently, RGR has more room to benefit methods whose optimization induces stronger changes in the representation space; (3) The improvements also vary across AIME24, AIME25, and HMMT25. This variation may be attributed to the different difficulty levels and reasoning characteristics of these benchmarks, which result in different distributions of correct and incorrect responses during self-improvement. Overall, RGR remains beneficial across diverse mathematical reasoning benchmarks, while its improvement varies with the characteristics of the target dataset.
\subsection{Ablation Study (RQ2)}

To investigate the contribution of each component to RGR, we conduct an ablation study with the following variants:
\begin{itemize} [leftmargin=2em]
    \item \textbf{w/o Ent}: We remove the $\mathcal{L}_{\mathrm{ent}}$ in Eq.(\ref{equ_ab}), removing the constraint on output distribution.
    \item \textbf{w/o Geo}: We remove the $\mathcal{L}_{\mathrm{geo}}$ in Eq.(\ref{equ_ab}), eliminating the constraint on intra-cluster compactness and inter-cluster separation.
\end{itemize}

\begin{wraptable}{r}{0.42\columnwidth}
\centering

\caption{Ablation study of RGR.}
\label{tab:loss_ablation}
\small
\setlength{\tabcolsep}{3pt}
\begin{tabular}{lccc}
\toprule
Method & AIME24 & AIME25 & HMMT25 \\
\midrule
OPSD & 56.67 & 43.06 & 28.89 \\
w/o Ent & 57.22 & 41.39 & 28.33 \\
w/o Geo & 58.61 & 42.22 & 29.17 \\
RGR & \textbf{59.89} & \textbf{43.60} & \textbf{30.00} \\
\bottomrule
\end{tabular}
\end{wraptable}
The results are shown in Table~\ref{tab:loss_ablation} that RGR achieves the best performance across all three benchmarks. From Table~\ref{tab:loss_ablation}, we can also observe that: (1) Removing the geometry constraint (\textbf{w/o Geo}) consistently degrades the performance across all benchmarks, compared with the full RGR. 
This indicates that explicitly preserving intra-cluster compactness and inter-cluster center separation is crucial for maintaining a structured representation space during OPSD optimization. Without this constraint, the hidden representations may become less organized, leading to limited capability improvement.

\begin{wrapfigure}[14]{R}{0.4\columnwidth}
    \centering
    \vspace{-15pt}
    \includegraphics[
        width=\linewidth,
        keepaspectratio
    ]{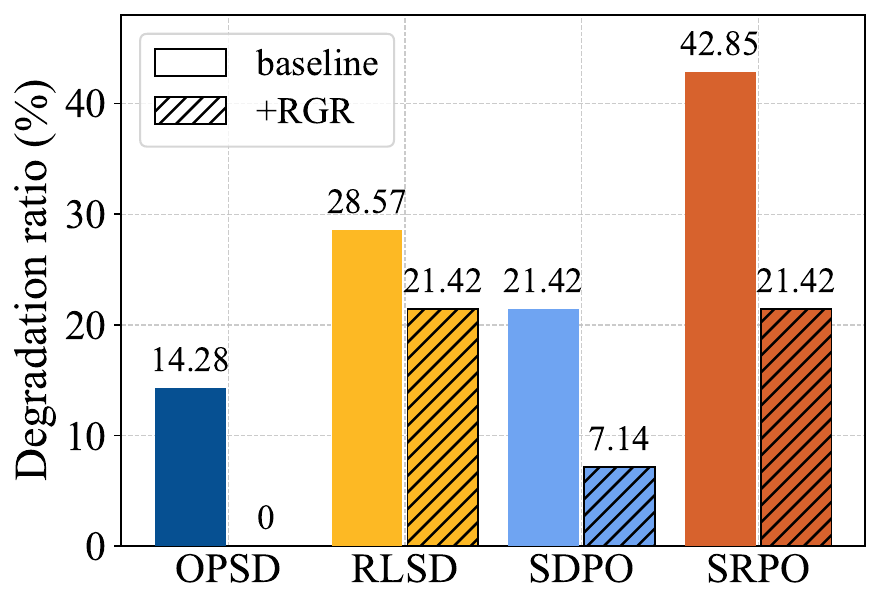}
    \caption{Comparison of the capability degradation ratio between the baseline and the baseline+RGR model on AIME24 (\%).
    }
    \label{fig:comparison}
    \vspace{-10pt}
\end{wrapfigure}
\begin{figure}[t!]
  \centering
  
  \includegraphics[width=\linewidth]{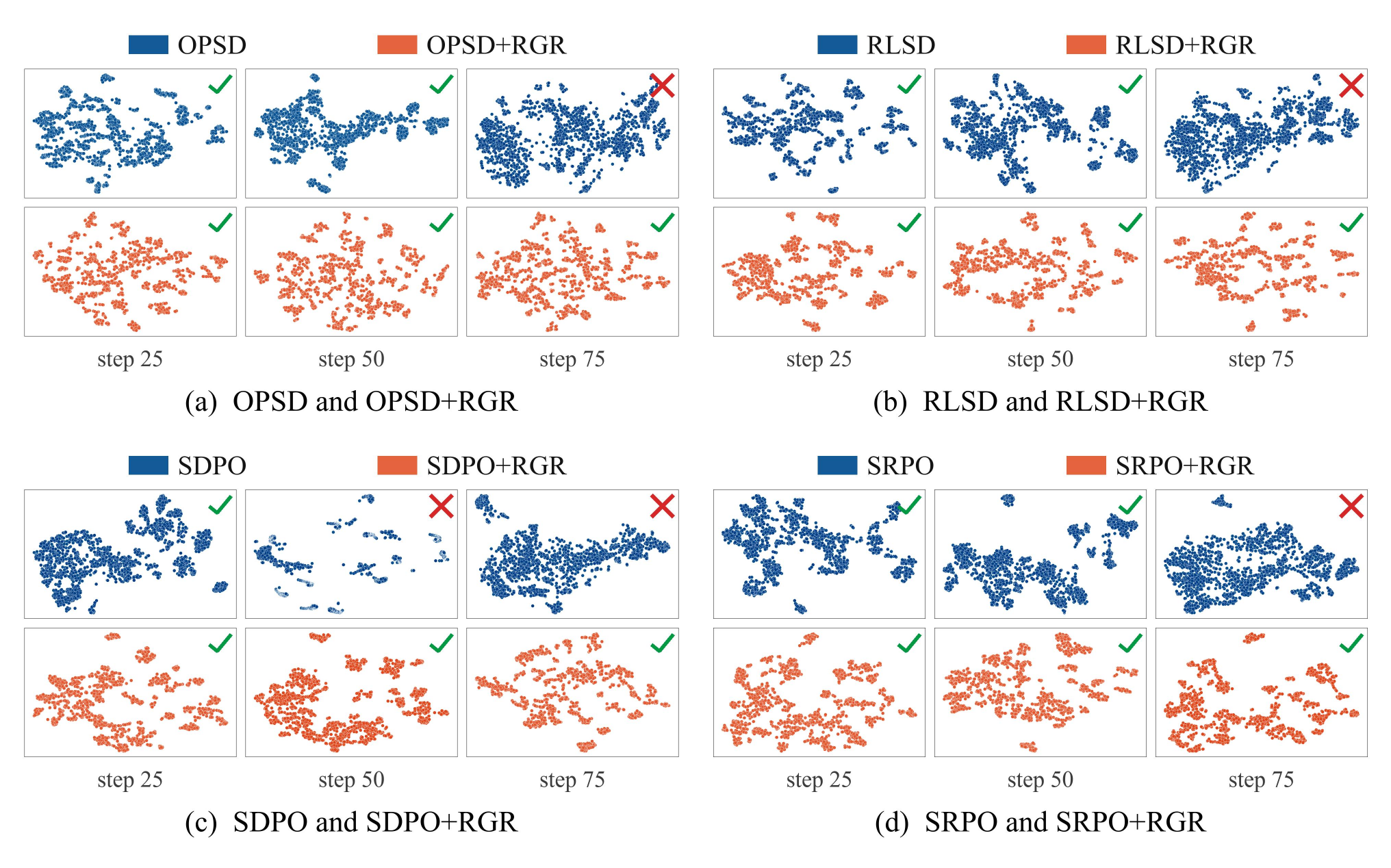}
   \vspace{-10pt}
  \caption{ Visualization of hidden state representations across different models and training steps, with response correctness indicated. }
  \vspace{-15pt}

  \label{fig:demo}
\end{figure}

(2) Removing the entropy constraint (\textbf{w/o Ent}) also leads to performance drops on one benchmarks. This suggests that entropy regularization provides complementary optimization benefits by further stabilizing the output distribution during self-distillation. However, the performance gap between RGR and \textbf{w/o Geo} is larger than that between RGR and \textbf{w/o Ent}, indicating that representation geometry preservation serves as the primary contributor to the improvements.

\vspace{-0.2cm}
\subsection{Training Dynamics Visualization (RQ3)}
\vspace{-0.2cm}

\textbf{Effect of RGR on mitigating progressive capability degradation.} To evaluate whether RGR alleviates progressive capability degradation, we
compare the ratio of queries that transition from correct to incorrect responses before and after incorporating RGR into different OPSD variants. As shown in Figure~\ref{fig:comparison}, adding RGR consistently reduces the degradation ratio across all evaluated methods. Specifically, RGR eliminates the degradation in OPSD and reduces the degradation ratios of SDPO, RLSD, and SRPO by 14.28\%, 7.15\%, and 21.43\%, respectively. These results demonstrate that RGR not only improves self-improvement performance but also helps preserve previously acquired capabilities by
maintaining a more stable hidden representation geometry during optimization.

\textbf{Representation geometry analysis.}
To further understand how RGR alleviates progressive capability degradation, we
visualize the hidden representations of responses corresponding to the
degraded queries. Specifically, we select queries that are initially responded
correctly but become incorrectly responded during optimization, and compare
the representation geometry before and after incorporating RGR.
As shown in Figure~\ref{fig:demo}, without RGR, the hidden
representations gradually lose their geometric structure, exhibiting reduced
inter-cluster center separation and less distinguishable cluster boundaries
when the response changes from correct to incorrect. In contrast, RGR
maintains more compact clusters with clearer boundaries  in most cases. These results are consistent with our theoretical analysis and
demonstrate that preserving representation geometry helps stabilize the
evolution of hidden representations during OPSD training.

\begin{wraptable}[16]{r}{0.42\columnwidth}
\centering
\small
\setlength{\tabcolsep}{2pt}
\renewcommand{\arraystretch}{1.0}
\vspace{ -10pt}
\caption{Forgetting evaluation on AIME24 (\%). \textbf{Before} denotes the performance before training on the tool-use dataset, while \textbf{After} denotes the performance after tool-use training.}
\label{tab:forgetting_evaluation}
\begin{tabular}{lccc}
\toprule
\textbf{Method}
& \textbf{Before}
& \textbf{After}
& \textbf{Forgetting} $\downarrow$ \\
\midrule
OPSD
& 56.67 & 54.72 & 1.95 \\
OPSD + RGR
& 59.89 & 58.06 & 1.83 \\ \hline
SDPO
& 38.61 & 33.33 & 5.28 \\
SDPO + RGR
& 41.94 & 40.00 & 1.94 \\ \hline
RLSD
& 50.27 & 47.22 & 3.05 \\
RLSD + RGR
& 51.39 & 48.00 & 3.39 \\ \hline
SRPO
& 37.78 & 34.72 & 3.06 \\
SRPO + RGR
& 39.44 & 38.33 & 1.11 \\
\bottomrule
\end{tabular}
\vspace{-30pt}
\end{wraptable}

Overall, RGR achieves consistent reductions in capability degradation across OPSD,
SDPO, RLSD, and SRPO, despite these methods adopting different
self-distillation objectives. This suggests that RGR
serves as a general geometry-preserving regularization mechanism that can be
integrated into different OPSD methods.

\subsection{Cross-Task Forgetting Study (RQ4)}

To further investigate whether RGR can mitigate capability degradation beyond the original task, we conduct a cross-task forgetting experiment. 
We first post-train a model on OpenThoughts to acquire mathematical problem-solving capabilities, and evaluate the model on AIME24 to measure performance. 
Then we continue post-training it on a tool-use dataset\citep{tang2023toolalpaca}, after convergence, we evaluate the model again on AIME24 to measure the performance after tool-use training. 
We apply this setting to OPSD, SDPO, RLSD, and SRPO, and compare their performance before and after training on the tool-use dataset, as shown in Table~\ref{tab:forgetting_evaluation}.
From Table~\ref{tab:forgetting_evaluation}, we can observe that: (1) Sequential training on a new task leads to forgetting of previously acquired capabilities. After tool-use training, all four methods exhibit performance degradation on AIME24, consistent with previous studies of catastrophic forgetting in continual fine-tuning of LLMs \citep{luo2025empirical,li2024revisiting}. 
(2) RGR generally alleviates cross-task forgetting. RGR reduces the forgetting gap for OPSD, SDPO, and SRPO, with particularly notable improvements for SDPO (5.28 $\rightarrow$ 1.94) and SRPO (3.06 $\rightarrow$ 1.11). This suggests that preserving representation geometry during subsequent optimization can help retain previously acquired capabilities. 
(3) The effectiveness of RGR depends on the underlying post-training method. For RLSD, the forgetting gap slightly increases from 3.05 to 3.39 after introducing RGR, suggesting that the effect of representation geometry preservation may vary across different optimization dynamics. Nevertheless, RGR reduces forgetting for three of the four methods, supporting its broader applicability to sequential post-training.

\section{Conclusion}

In this work, we identify and investigate the progressive capability degradation phenomenon in on-policy self-distillation, where LLM may lose previously acquired capabilities during self-improvement despite overall performance improvements. Through empirical analysis and theoretical investigation, we reveal that this degradation is associated with the deterioration of hidden representation geometry, characterized by reduced intra-cluster compactness and inter-cluster center separation. Based on these findings, we propose Representation Geometry Regularization (RGR), a general regularization framework that preserves intra-cluster compactness and inter-cluster center separation during OPSD optimization. Extensive experiments across multiple OPSD variants and mathematical reasoning benchmarks demonstrate that RGR effectively reduces correct-to-incorrect transitions while maintaining the self-improvement capability of OPSD. Future work will explore more adaptive strategies for preserving representation geometry and extending the proposed framework to broader model optimization paradigms beyond self-distillation.


\clearpage

\appendix
\section{Sensitivity Analysis}

\paragraph{Hyperparameter sensitivity.}
We investigate the sensitivity of our method to the
hyperparameters $\alpha$ and $\beta$ in
Eq.~(\ref{equ_ab}). We vary one hyperparameter at a time while
fixing the other at its selected value. Specifically, we vary
$\beta$ over
$\{2.5, 4.0, 5.0, 6.0\}\times10^{-2}$ while fixing
$\alpha=2.5\times10^{-2}$. Similarly, we vary $\alpha$ over
$\{0.5, 2.5, 3.5, 4.0, 5.0\}\times10^{-2}$ while fixing
$\beta=5\times10^{-2}$. As shown in
Figure~\ref{fig:hyperparameter}, 
As shown in the figure, the performance of our model follows a smooth increasing-then-decreasing trend with respect to the hyperparameter values. This observation suggests that the proposed method does not require highly precise hyperparameter tuning and remains stable change across a relatively broad range of parameter settings.

\begin{figure*}[ht]
    \centering
    \includegraphics[width=\textwidth]
    {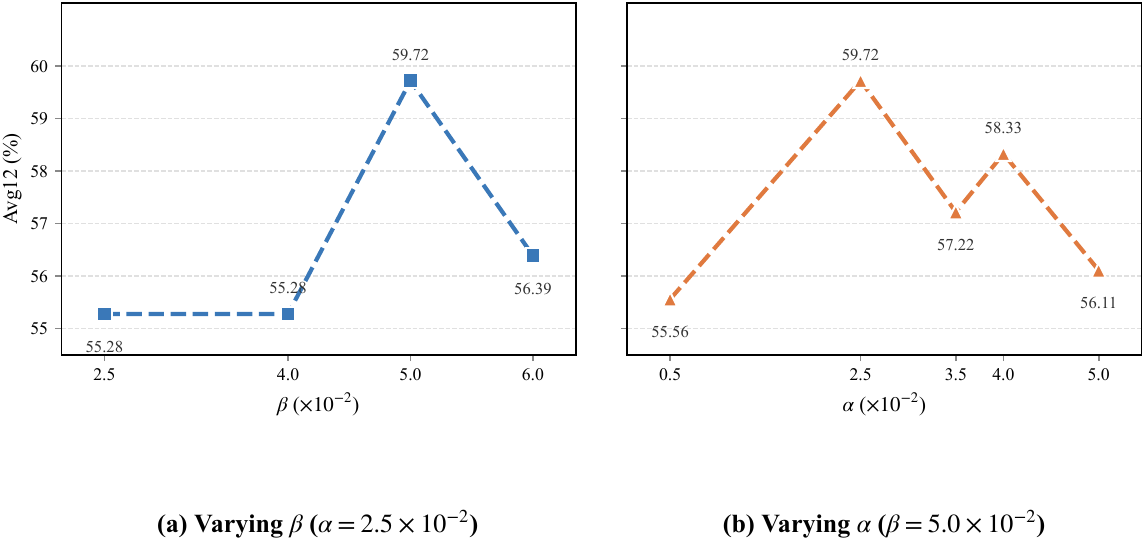}
    \caption{
    Sensitivity analysis of the hyperparameters $\alpha$ and
    $\beta$ on AIME24.
    Left: performance under different values of $\beta$, with
    $\alpha$ fixed at $2.5\times10^{-2}$.
    Right: performance under different values of $\alpha$, with
    $\beta$ fixed at $5\times10^{-2}$.
    The stars indicate the selected setting:
    $\alpha=2.5\times10^{-2}$ and
    $\beta=5\times10^{-2}$.
    }
    \label{fig:hyperparameter}
\end{figure*}

\section{Implementation Details}
\label{app:implementation}

Table~\ref{tab:training_config} lists the training hyperparameters that are
shared across all experiments. All runs use LoRA for parameter-efficient
fine-tuning with rank $r{=}64$ and $\alpha{=}128$, and are optimized with a
learning rate of $5\times10^{-6}$ under a maximum gradient norm of $0.1$.
The maximum sequence length is $20{,}000$ tokens, while completions are
capped at $1{,}024$ tokens. Rollout sampling uses a temperature of $1.1$,
top-$p$ of $0.95$, and top-$k$ of $20$. Training is performed in BF16
precision.

\begin{table}[h]
\centering
\caption{Hyperparameters used across all experiments.}
\label{tab:training_config}
\small
\begin{tabular}{ll}
\toprule
\textbf{Hyperparameter} & \textbf{Value} \\
\midrule
Learning rate & $5\times10^{-6}$ \\
Max gradient norm & 0.1 \\
LoRA rank $r$ & 64 \\
LoRA $\alpha$ & 128 \\
Max sequence length & 20{,}000 \\
Max completion length & 1{,}024 \\
Temperature & 1.1 \\
Top-$p$ & 0.95 \\
Top-$k$ & 20 \\
BF16 & Enabled \\
\bottomrule
\end{tabular}
\end{table}

\section{Notation and Additional Preliminaries}
\label{app:notation}

\begin{table}[t]
    \centering
    \caption{Summary of the main notation used in the paper.}
    \label{tab:notation}
    \small
    \renewcommand{\arraystretch}{1.15}
    \setlength{\tabcolsep}{6pt}
    \begin{tabularx}{\linewidth}{@{}lX@{}}
        \toprule
        Symbol & Description \\
\midrule

        $q$
        & a query \\

        $\pi_\theta$
        & LLM policy parameterized by $\theta$ and trained within the OPSD framework. \\

       $\theta$
        & Parameters of the LLM policy. \\
        
        $y_t$
        & Token generated at decoding step $t$ by $\pi_\theta$. \\
        
        $\mathbf{y}=(y_1,\ldots,y_T)$
        & Response consisting of $T$ tokens generated by $\pi_\theta$. \\

    $\mathbf{y}_{<t}=(y_1,\ldots,y_{t-1})$
        & Sequence of tokens generated before step $t$. \\

        $\pi_\theta(\cdot\mid q,\mathbf{y}_{<t})$
        & Conditional distribution over the next token given the query
          and previously generated tokens. \\

        $T$
        & Length of the generated response. \\

        $\mathcal{H}=\{\mathbf{h}_1,\ldots,\mathbf{h}_T\}$
        & the hidden representations corresponding to the generated response. \\

        $\hat{\mathcal{H}}=\{\mathbf{\hat{h}}_1,\ldots,\mathbf{\hat{h}}_T\}$
        & the hidden representations generated by the teacher model. \\

        $\mathcal{H}^*=\{\mathbf{h}_1^*,\ldots,\mathbf{h}_T^*\}$
        & the hidden representations generated by the optimal teacher model. \\

        ${\mathcal{H}}^{\dagger}=\{\frac{\mathbf{h}_1^*+\mathbf{\hat{h}}_1}{2}+\ldots,\frac{\mathbf{h}_T^*+\mathbf{\hat{h}}_T}{2}\}$
        & the mixed hidden representations of $\mathcal{H}$ and $\hat{\mathcal{H}}$. \\

        $\mathcal{H}^{ref}=\{\mathbf{h}_1^{ref},\ldots,\mathbf{h}_T^{ref}\}$
        & the hidden representations of the reference model. \\

        $\mathcal{H}=\{\mathbf{h}_1,\ldots,\mathbf{h}_T\}$
        & the hidden representations corresponding to the generated response. \\

        $\mathcal{H}=\{\mathbf{h}_1,\ldots,\mathbf{h}_T\}$
        & the hidden representations corresponding to the generated response. \\

        $\mathcal{C}=\{C_1,\ldots,C_K\}$
        & $\mathcal{H}$ into $K$ representation clusters. \\

        $\mathcal{C^\dagger}=\{C_1^\dagger,\ldots,C_K^\dagger\}$,
        & $\mathcal{H^{\dagger}}$ into $K$ representation clusters. \\

        $\mathcal{C}^{ref}=\{C_1^{ref},\ldots,C_K^{ref}\}$,
        & $\mathcal{H}^{ref}$ into $K$ representation clusters. \\

        $\mathcal{X}_k$ := $\{\mathbf{h}_i \mid i \in \mathcal{I}_k^{\dagger}\}$
        & the associated index set of  for each cluster $C_k^\dagger$ \\
        \\

        $|\mathcal{C}|$
        & the number of clusters of $\mathcal{C}$. \\

        $|\mathcal{C}^{\dagger}|$
        & the number of clusters of $\mathcal{C}^{\dagger}$. \\
        
        $|\mathcal{C}_\mathrm{ref}|$
        & the number of clusters of $\mathcal{C}_\mathrm{ref}$. \\

        $\mathbf{H}=
[\mathbf{h}_1,\ldots,\mathbf{h}_T] \in\mathbb{R}^{T\times d}$
        & the hidden representation matrix of the student model. \\

        $\mathbf{H}^{ref}=
[\mathbf{h}_1^{ref},\ldots,\mathbf{h}_T^{ref}] \in\mathbb{R}^{T\times d}$
        & the hidden representation matrix of the reference model. \\

        $\mu_k$
        & $\frac{1}{|C_k|} \sum_{\mathbf{h}_i\in C_k} \mathbf{h}_i$. \\

        $\mu_k^\dagger$
        & $\frac{1}{|C_k^\dagger|} \sum_{\mathbf{h}_i^\dagger\in C_k^\dagger} \mathbf{h}_i^\dagger$. \\

        $\mu_k^{ref}$
        & $\frac{1}{|C_k^{ref}|} \sum_{\mathbf{h}_i^{ref}\in C_k^{ref}} \mathbf{h}_i^{ref}$. \\

$\mathbf{L}$
        & The Laplacian matrix computed from 
the hidden representation matrix of the student model with a Gaussian kernel. \\

$\mathbf{L}^{\mathrm{ref}}$
        & The Laplacian matrix computed from 
the hidden representation matrix of the reference model with a Gaussian kernel. \\

        $\boldsymbol{\lambda}
=
\left[
\lambda_1(\mathbf{L}),\ldots,
\lambda_n(\mathbf{L})
\right]^\top$
        & the eigenvalues of the Laplacian matrix $\mathbf{L}$. \\

$\boldsymbol{\lambda}^{\mathrm{ref}}
=
\left[
\lambda_1(\mathbf{L}^{\mathrm{ref}}),\ldots,
\lambda_n(\mathbf{L}^{\mathrm{ref}})
\right]^\top$ & the eigenvalues of the Laplacian matrix $\mathbf{L}^{\mathrm{ref}}$. \\

        $\mu_k^{ref}$
        & $\frac{1}{|C_k^{ref}|} \sum_{\mathbf{h}_i^{ref}\in C_k^{ref}} \mathbf{h}_i^{ref}$. \\

        $\mathcal{V}$
        & the vocabulary of the LLM. \\

        $|\mathcal{V}|$
        & the size of the vocabulary. \\

        $W\in \mathbb{R}^{|\mathcal{V}|\times d}$
        & the weight matrix of the final layer for the trained model. \\

        $I_{|\mathcal{V}|} \in \mathbb{R}^{|\mathcal{V}|\times |\mathcal{V}|}$
        & the identity matrix . \\
        
        $\mathbf{1}_{|\mathcal{V}|} \in \mathbb{R}^{|\mathcal{V}|}$
        &  the all-ones column vector. \\
        
        $|\Delta|$
        &  $\sqrt{\sum_i||\hat{\mathbf{h}_i}-{\mathbf{h}_i}^*||_2}$. \\     
        
        $D(\cdot|\cdot)$
         &  a divergence measure between the teacher and student token distributions. \\    

         $\mathcal{E}_{\text{KD}}(\pi_{\theta})$
          &  a divergence measure between the teacher and student token distributions. \\
\bottomrule
\end{tabularx}
\end{table}

\section{Proof of Theorem~\ref{thm:excess_risk}}
\label{proof:excess_risk}
By definition~\ref{def:excess_risk}, we obtain
\begin{equation}
\mathcal{E}_{\text{KD}}(\pi_\theta)=\mathbb{E}_{q \sim \mathcal{D}, y \sim \pi_{\theta}(\cdot \mid q)}
\left[
\sum_{t=1}^{T}
D\left(
\pi_{\theta^*}(\cdot \mid q, z, y_{<t})
\,\middle\|\,
\pi_{\theta}(\cdot \mid q, y_{<t})
\right)\\
-D\left(
\pi_{\theta^*}(\cdot \mid q, z, y_{<t})
\,\middle\|\,
\pi_{\theta^*}(\cdot \mid q, y_{<t})
\right)
\right]
\end{equation}

To simplify the notation, we define
\begin{equation*}
q^*_t:=\pi_{\theta^*}(\cdot \mid q, z, y_{<t})
\end{equation*}
\begin{equation*}
p_t:=\pi_{\theta}(\cdot \mid q, y_{<t})
\end{equation*}
\begin{equation*}
q_t:=\pi_{\theta}(\cdot \mid q, z, y_{<t})
\end{equation*}
\begin{equation*}
p^*_t:=\pi_{\theta^*}(\cdot \mid q, y_{<t})
\end{equation*}

Let $z_t$ denote the corresponding logits of $p_t$, $z_t^*$ denote the corresponding logits of $q_t^*$ and $\hat{z_t}$ denotes the corresponding logit of $q_t$. For each $f\in\mathcal{F}_{OPSD}$,
let $p_f$ denote the policy induced by $f$.

For one response and $q \sim \mathcal{D}, y \sim \pi_{\theta}$, let $\mathcal{E}_{q,\theta}=\sum_t (KL(q^*_t||p_t))-E(KL(q^*_t||p^*_t)$, we have 
\begin{equation*}
\mathcal{E}_{\text{KD}}(\pi_{\theta})=\mathbb{E}_{q \sim \mathcal{D}, y \sim \pi_{\theta}}\mathcal{E}_{q,\theta}
\end{equation*}

\begin{equation*}
\mathcal{E}_{q,\theta}=\sum_t (KL(q^*_t||p_t))-KL(q^*_t||p^*_t))
\end{equation*}
\begin{equation*}
\mathcal{E}_{q,\theta}=\sum_t [KL(q^*_t||p_t)-KL(q_t||p_t)+KL(q_t||p_t)-KL(q_t^*||q_t)+KL(q_t^*||q_t)-KL(q^*_t||p^*_t)]
\end{equation*}

\begin{equation*}
\sum_t [KL(q^*_t||p_t)-KL(q_t||p_t)]\leq sup_{\pi_{\theta}\in {\Pi_{\text{OPSD}}}}|\sum_t[KL(q^*_t||p_f)-KL(q_t||p_f)]|
\end{equation*}
\begin{equation*}
\sum_t(KL(q_t||p_t)-KL(q_t^*||q_t))\leq0
\end{equation*}
\begin{equation*}
\sum_t(KL(q_t^*||q_t)-KL(q^*_t||p^*_t))\leq \sup_{\pi_{\theta}\in{\Pi_{\text{OPSD}}}}|\sum_t[KL(q^*_t||p_f)-KL(q_t||p_f)]|
\end{equation*}
so we have 

\begin{equation*}
\mathcal{E}_{q,\theta}\leq 2*\sup_{\pi_{\theta}\in{\Pi_{\text{OPSD}}}}|\sum_t[KL(q^*_t||p_f)-KL(q_t||p_f)]|
\end{equation*}

\begin{equation}
\label{equ_1}
\mathcal{E}_{q,\theta}\leq 2*\sup_{\pi_{\theta}\in{\Pi_{\text{OPSD}}}}|\sum_t[\sum_{i} q^*_{t,i}logq^*_{t,i}-q^*_{t,i} logp_{t,i}-\hat{q}_{t,i}logq^*_{t,i}+\hat{q}_{t,i} logp_{t,i}]|
\end{equation}

Consider an arbitrary LLM and fix an input context. Let
\(\widetilde{\mathbf z}
=(\widetilde z_1,\ldots,\widetilde z_V)^\top\in\mathbb{R}^V\)
denote the logit vector produced by the model, where \(V\) is the
vocabulary size. The corresponding next-token probability vector
\(\mathbf p=(p_1,\ldots,p_V)^\top\) is given by
\[
    p_i
    :=
    \frac{\exp(\widetilde z_i)}
         {\sum_{j=1}^{V}\exp(\widetilde z_j)},
    \qquad i\in[V].
\]

Let \(\mathbf 1\in\mathbb{R}^V\) denote the all-ones vector. Since
softmax is invariant to a constant shift, for any \(c\in\mathbb{R}\),
\[
    \operatorname{softmax}
    (\widetilde{\mathbf z}+c\mathbf 1)
    =
    \operatorname{softmax}(\widetilde{\mathbf z}).
\]
Define the centering matrix
\[
    A
    :=
    I_V-\frac{1}{V}\mathbf 1\mathbf 1^\top
\]
and the centered logit vector
\[
    \mathbf z
    :=
    A\widetilde{\mathbf z}
    =
    \widetilde{\mathbf z}
    -
    \left(
        \frac{1}{V}\mathbf 1^\top\widetilde{\mathbf z}
    \right)\mathbf 1.
\]
Then
\[
    \mathbf 1^\top\mathbf z=0,
    \qquad
    A\mathbf z=\mathbf z,
\]
and
\[
    \operatorname{softmax}(\mathbf z)
    =
    \operatorname{softmax}(\widetilde{\mathbf z}).
\]
Therefore, without loss of generality, we may work with centered
logits and assume throughout that
\[
    \mathbf z=A\mathbf z,
    \qquad\text{or equivalently}\qquad
    \sum_{j=1}^{V}z_j=0.
\]
For centered logits satisfying
\(\lVert\mathbf z\rVert_2\to 0\), Taylor expansion gives
\[
    e^{z_i}
    =
    1+z_i+\frac{1}{2}z_i^2
    +O\!\left(\lVert\mathbf z\rVert_2^3\right).
\]
Consequently,
\begin{align*}
    \sum_{j=1}^{V}e^{z_j}
    &=
    V+\sum_{j=1}^{V}z_j
    +\frac{1}{2}\sum_{j=1}^{V}z_j^2
    +O\!\left(\lVert\mathbf z\rVert_2^3\right) \\
    &=
    V+\frac{1}{2}\lVert\mathbf z\rVert_2^2
    +O\!\left(\lVert\mathbf z\rVert_2^3\right),
\end{align*}
where the second equality follows from
\(\sum_{j=1}^{V}z_j=0\). Therefore,
\begin{align*}
    \log\left(\sum_{j=1}^{V}e^{z_j}\right)
    &=
    \log V
    +
    \log\left(
        1+\frac{1}{2V}\lVert\mathbf z\rVert_2^2
        +O\!\left(\lVert\mathbf z\rVert_2^3\right)
    \right) \\
    &=
    \log V
    +\frac{1}{2V}\lVert\mathbf z\rVert_2^2
    +O\!\left(\lVert\mathbf z\rVert_2^3\right).
\end{align*}
In particular,
\[
    p_i
    =
    \frac{1}{V}
    \left(
        1+z_i+O(\lVert\mathbf z\rVert_2^2)
    \right).
\]

For Qwen3, let \(\mathbf h\in\mathbb{R}^d\) denote the final hidden
representation at the current token position and let
\(W\in\mathbb{R}^{V\times d}\) denote the weight matrix of its
bias-free language-modeling head. The original logits are
\(\widetilde{\mathbf z}=W\mathbf h\). Defining
\[
    A:=I_V-\frac{1}{V}\mathbf 1\mathbf 1^\top,
\]
the centered logits can be written as
\[
    \mathbf z:=A\widetilde{\mathbf z}=AW\mathbf h.
\]
Since \(\mathbf 1^\top\mathbf z=0\) and softmax is invariant to
constant shifts, \(\mathbf z\) induces the same next-token
distribution as \(\widetilde{\mathbf z}\).

Since OPSD employs LoRA without applying it to the output layer, the output-layer weights remain unchanged. Therefore, we obtain:

So we can simlify~\ref{equ_1}:

\begin{equation}
\mathcal{E}_{q,\theta}\leq \frac{2}{|\mathcal{V}|}*\sup_{\pi_{\theta}\in{\Pi_{\text{OPSD}}}}|\sum_t[(h_t-\frac{h_t^*+\hat{h_t}}{2})^{T}W^T(I_{|\mathcal{V}|}-\frac{1}{|\mathcal{V}|}1_{\mathcal{|V|}}1_{\mathcal{|V|}}^T)W(h^*_t-\hat{h}_t)]|
\end{equation}

Throughout our analysis, we assume that the hidden representations are normalized using RMSNorm without learnable scaling parameters.

By the Cauchy–Schwarz inequality, we have

\begin{equation}
\begin{aligned}
\mathcal{E}_{\text{KD}}(\pi_{\theta}) \leq \frac{2}{N|\mathcal{V}|}\sum_N\sup_{\pi_{\theta}\in{\Pi_{\text{OPSD}}}}|\Delta|||M||_2[\sqrt{\frac{Td}{2}+\frac{1}{2}\sum_{|\mathcal{C}^\dagger|}\sum_{{\mathbf{h}_i}\in \mathcal{X}_k}||\mathbf{h}_i-\mu_k^{\mathcal{X}}||_2^2-\frac{1}{2}\sum_{|\mathcal{C}^\dagger|}|\mathcal{C}^\dagger_k|||\mu_k^{\mathcal{X}}-\mu||^2_2}+\\
\sqrt{\sum_{|\mathcal{C}^\dagger|}|\mathcal{C}_k^\dagger|||\mu_k^{\mathcal{X}}-\mu_k^{\dagger}||^2_2+\sum_{|\mathcal{C}^\dagger|}\sum_{{\mathbf{h}_i^{\dagger}}\in \mathcal{X}_k}||\mathbf{h}_i^{\dagger}-\mu_k^{\dagger}||_2^2}]
\end{aligned}
\end{equation}

\section{Proof of Theorem~\ref{thm:cluster}}
\label{proof:cluster}
\begin{proof}

For each cluster, we stack all 
$h_{i,ref} \in \mathcal{C}_k^{ref}$
to form
$H_k^{ref}=[(h_1^{ref})^\top,\ldots,(h_{|C_k^{ref}|}^{ref})^\top]^\top$,
then we define
$G_k^{ref}=H_k^{ref}H_k^{ref\top}$.
Accordingly, we stack all 
$h_{i,align} \in \mathcal{C}_k^{align}$
to form
$H_k^{align}=[(h_1^{align})^\top,\ldots,(h_{|C_k^{align}|}^{align})^\top]^\top$,
then we define
$G_k^{align}=H_k^{align}H_k^{align\top}$.

and $\mathbf{1}_{|C_k^{ref}|} \in \mathbb{R}^{|C_k^{ref}|}$ denotes the all-ones column vector.$|\mathcal{C}_{ref}|$ denotes the number of clusters in the reference model.

First, we can simplify the optimization objective of intra-cluster compactness for each cluster:

\begin{equation}
\label{eq:intra_ini}
\left|
\sum_{
\mathbf{h}_{i}\in C_{k}^{align},
}
\left\|
\mathbf{h}_{i}-\boldsymbol{\mu}_{k}^{align}
\right\|_2^2
-
\sum_{\mathbf{h}_i^\mathrm{ref}\in C_k^\mathrm{ref}}
\left\|
\mathbf{h}_i^\mathrm{ref}-\boldsymbol{\mu}_k^\mathrm{ref}
\right\|_2^2
\right|,
\leq||I_{|C_{ref}|}-\mathbf{1}_{|C^{ref}_k|}^\top\mathbf{1}_{|C^{ref}_k|}||_F||G^{ref}_k-G_{k}||_F
\end{equation}

Second, $G_k^{align}$ is a principal submatrix of $HH^T$ and $G_k^{ref}$ is a principal submatrix of $HH^T$. Because the squared Frobenius norm is additive over pairwise disjoint principal submatrices, 
Therefore, for all clusters, we can obtain this through the Cauchy-Schwarz inequality:
\[
\sum_{k} ||I_{|C_k^{ref}|}-\mathbf{1}_{|C_k^{ref}|}^\top\mathbf{1}_{|C_k^{ref}|}||_F||G_{k}^{ref}-G_k^{align}||_F\leq (T-|C_{ref}|)^{\frac{1}{2}}||\sum_k (G_k^{ref}-G_k^{align})||_F
\]

Third, through the inclusion of the pairwise disjoint principal submatrices, we can get the upper bound of $\mathcal{L}_{intra}$:

\begin{align}\leq(T-|\mathcal{C}_{ref}|)^{\frac{1}{2}}||HH^T-H_{ref}H_{ref}^\top||_F
\end{align}
\end{proof}

\section{Proof of Proposition~\ref{thm:inter}}
\label{proof:inter}
\begin{proof}
Through the definition of the Laplacian matrix and the Cauchy-Schwarz inequality, we can obtain:

Let $\mathbf{1}_{T}\in \mathbb{R^T}$ denote the all-ones column vector.

\begin{equation*}
\begin{aligned}
    ||L-L^{ref}||_F=||diag(A\mathbf{1}_{T})-diag(A^{ref}\mathbf{1}_{T})+A-A^{ref}||_F\\\leq||diag(A-A^{ref})\mathbf{1}_{T}||_F+||A-A^{ref}||_F,
\end{aligned}
\end{equation*}

\begin{equation*}
||diag(A-A^{ref})\mathbf{1}_{T}||_F=||A-A^{ref}||_2\leq\sqrt{T}||A-A^{ref}||_F
\end{equation*}
\begin{equation*}
||diag(A-A^{ref})\mathbf{1}_{T}||_F+||A-A^{ref}||_F\leq(\sqrt{T}+1)||A-A^{ref}||_F
\end{equation*}
\begin{equation*}
(\sqrt{T}+1)||A-A^{ref}||_F\leq\frac{1+\sqrt{T}}{2\sigma^2}||D-D^{ref}||_F
\end{equation*}

\begin{equation*}
||D-D^{ref}||_F=||diag\{HH^T\}\mathbf{1}_{T}^T+\mathbf{1}_{T}diag\{HH^T\}-2HH^T-diag\{H_{ref}H_{ref}^T\}1^T-1diag\{H_{ref}H_{ref}^T\}+2H_{ref}H_{ref}^T||_F
\end{equation*}
Through Cauchy-Schwarz inequality, we have

\begin{equation*}
\sqrt{\sum_i |\lambda(L)-\lambda(L^{ref})|^2}\leq(\frac{1+\sqrt{T}}{\sigma})^2||HH^T-H_{ref}H_{ref}^T||_F
\end{equation*}
\end{proof}

\end{document}